\documentclass{article}

     \PassOptionsToPackage{round, colon}{natbib}

     \usepackage[preprint]{neurips_2019}

\usepackage[utf8]{inputenc} 
\usepackage[T1]{fontenc}    
\usepackage{hyperref}       
\usepackage{url}            
\usepackage{booktabs}       
\usepackage{amsfonts}       
\usepackage{nicefrac}       
\usepackage{microtype}      

\usepackage{amsmath}
\usepackage{amsthm}
\newtheorem{theorem}{Theorem}[section]

\newtheorem{lemma}[theorem]{Lemma}
\newtheorem{definition}[theorem]{Definition}
\usepackage{multirow}
\usepackage{xfrac}

\title{Bandlimiting Neural Networks Against \\ Adversarial Attacks}

\author{%
  Yuping Lin \hspace{0.2cm} Kasra Ahmadi K. A. \hspace{0.2cm} Hui Jiang\\
    iFLYTEK Laboratory for Neural Computing and Machine Learning (iNCML) \\
    Department of Electrical Engineering and Computer Science  \\
  York University,  4700 Keele Street, Toronto, Ontario, M3J 1P3, Canada \\     
  \texttt{\{yuping,kasraah,hj\}@eecs.yorku.ca} \\
}

\begin{document}

\maketitle

\begin{abstract}
  In this paper, we study the adversarial attack and defence problem in deep learning from the perspective of Fourier analysis. We first explicitly compute the Fourier transform of deep ReLU neural networks and show that there exist decaying but non-zero high frequency components in the Fourier spectrum of neural networks. We demonstrate that the vulnerability of neural networks towards adversarial samples can be attributed to these insignificant but non-zero high frequency components. Based on this analysis, we propose to use a simple {\em post-averaging} technique to smooth out these high frequency components to improve the robustness of neural networks against adversarial attacks. Experimental results on the ImageNet dataset have shown that our proposed method is universally effective to defend many existing adversarial attacking methods proposed in the literature, including FGSM, PGD, DeepFool and C\&W attacks. Our post-averaging method is simple since it does not require any re-training, and meanwhile it can successfully defend over 95\% of the adversarial samples generated by these methods without introducing any significant performance degradation (less than 1\%) on the original clean images. 
\end{abstract}

\section{Introduction}
Although deep neural networks (DNN) have shown to be powerful in many machine learning tasks, researchers~\citep{talk_adv} found that they are vulnerable to \textit{adversarial samples}. Adversarial samples are subtly altered inputs that can fool the trained model to produce erroneous outputs. They are more commonly seen in image classification task and typically the perturbations to the original images are so small that they are imperceptible to human eye.

Research in adversarial attacks and defences is highly active in recent years. In the attack side, many attacking methods have been proposed~\citep{talk_adv, FGSM, JSMA, PracticalAttack, DeepFool, BIM, PGD, CW, ZOO, genAttack, BoundaryAttack}, with various ways to generate effective adversarial samples to circumvent new proposed defence methods. However, since different attacks may be more effective to different defences or datasets, there is no consensus on which attack is the strongest. Hence for the sake of simplicity, in this work, we will evaluate our proposed defence approach against four popular and relatively strong attacks for empirical analysis. In the defence side, various defence mechanisms have also been proposed, including adversarial training~\citep{advTrain_1, BIM, advTrain_2, PGD}, network distillation~\citep{distillationNet}, gradient masking~\citep{gradMask}, adversarial detection~\citep{advDetect} and adding modifications to neural networks~\citep{randomLayer}. Nonetheless, many of them were quickly defeated by new types of attacks~\citep{defeated_1, defeated_2, defeated_3, CW, defeated_4, defeated_5, genAttack}. \citet{PGD} tried to provide a theoretical security guarantee for adversarial training by a min-max loss formulation, but the difficulties in non-convex optimization and finding the ultimate adversarial samples for training may loosen this robustness guarantee. As a result, so far there is no defence that is universally robust to all adversarial attacks.

Along the line of researches, there were also investigations into the properties and existence of adversarial samples. \citet{talk_adv} first observed the transferability of adversarial samples across models trained with different hyper-parameters and across different training sets. They also attributed the adversarial samples to the low-probability blind spots in the manifold. In~\citep{FGSM}, the authors explained adversarial samples as "a result of models being too linear, rather than too nonlinear." In a later paper, \citet{advTransfer} showed the transferability occurs across models with different structures and even different machine learning techniques in addition to neural networks. In summary, the general existence and transferability of adversarial samples are well known but the reason of adversarial vulnerability still needs further investigation.

In general, the observation that some small imperceptible perturbations in the inputs of neural networks lead to huge unexpected fluctuations in outputs must correspond to high frequency components in the Fourier spectrum of neutral networks. In this paper, we will start with the Fourier analysis of neural networks and elucidate why there always exist some decaying but nonzero high frequency response components in neural networks. Based on this analysis, we show neural networks are inherently vulnerable to adversarial samples due to the underlying model structure and why simple parameter regularization fails to solve this problem. Next, we propose a simple {\em post-averaging} method to tackle this problem. Our proposed method is fairly simple since it works as a post-processing stage of any given neural network models and it does not require to re-train neural networks at all. Furthermore, we have evaluated the post-averaging method against four popular adversarial attacking methods and our method is shown to be universally effective in defending all examined attacks. Experimental results on the ImageNet dataset have shown that our simple post-averaging method can successfully defend over 95\% of adversarial samples generated by these attacks with little performance degradation (less than 1\%) on the original clean images.

\section{Fourier analysis of neural networks}
\label{SEC:Fourier}

In order to understand the behaviour of adversarial samples, it is essential to find the Fourier transform of neural networks. Fortunately, for some widely used neural networks, namely fully-connected neural networks using ReLU activation functions, we may explicitly derive their Fourier transform under some minor conditions. As we will show, these theoretical results will shed light on how adversarial samples happen in neural networks. 

\subsection{Fourier transform of fully-connected ReLU neural networks}

As we know, any fully-connected ReLU neural networks (prior to the softmax layer)  essentially form piece-wise linear functions in input space. 
Due to space limit, we will only present the main results in this section and the proofs and more details may be found in Appendix. 

\begin{definition}
    A piece-wise linear function is a continuous function $f:\mathbb{R}^n \xrightarrow{} \mathbb{R}$ such that there are some hyperplanes passing through origin and dividing $\mathbb{R}^n$ into $M$ pairwise disjoint regions $\mathfrak{R}_m$, $(m=1,2,...,M)$, on each of which $f$ is linear:
\[ 
        f(\mathbf{x})=\left\{\begin{array}{ll}
            \mathbf{w}_1 \cdot \mathbf{x} & \mathbf{x}\in\mathfrak{R}_1 \\
            \mathbf{w}_2 \cdot \mathbf{x} & \mathbf{x}\in\mathfrak{R}_2 \\
            \vdots\\
            \mathbf{w}_M \cdot \mathbf{x} & \mathbf{x}\in\mathfrak{R}_M
            \end{array}\right.
\]        
\end{definition}
\begin{lemma}
Composition of a piece-wise linear function with a ReLU activation function is also a piece-wise linear function.
\end{lemma}
\begin{theorem}
\label{dnn-pwlinear}
    The output of any hidden unit in an unbiased fully-connected ReLU neural network is a piece-wise linear function.
\end{theorem}
This is straightforward because the input to any hidden node is a linear combination of piece-wise linear functions and this input is composed with the ReLU activation function to yield the output, which is also piece-wise linear. However,
each region $\mathfrak{R}_m$ is the intersection of a different number of half-spaces, enclosed by various hyperplanes in $\mathbb{R}^n$. In general, these regions $\mathfrak{R}_m$ $(m=1, \cdots,M)$ do not have simple shapes. For the purpose of mathematical analysis, we need to decompose each region into a union of some well-defined shapes having a uniform form,  which is called {\em infinite simplex}. 
\begin{definition}
    Let $\mathbf{V}=\{{\bf v}_1, {\bf v}_2,..., {\bf v}_n \}$ be a set of $n$ linearly independent vectors in $\mathbb{R}^n$. An infinite simplex, $\mathfrak{R}_\mathbf{V}^+$, is defined as the region linearly spanned by $\mathbf{V}$ using only positive weights:
    \begin{align}
        \mathfrak{R}_\mathbf{V}^+=\left\{\sum_{k=1}^n\ \alpha_k {\bf v}_k\ \bigg\vert \ \alpha_k>0, \;\;  k=1,2,\cdots,n \right\}
    \end{align}
\end{definition}
\begin{theorem}
    \label{suire}
    Each piece-wise linear function $f(\mathbf{x})$ can be formulated as a summation of some simpler functions: $ f(\mathbf{x}) = \sum_{l=1}^L f_l(\mathbf{x})$, each of which is linear and non-zero only in an infinite simplex as follows:
\begin{equation}
\label{pislif}
f_l(\mathbf{x}) =\left\{\begin{array}{lr}
            \mathbf{w}_l \cdot \mathbf{x} & \mathbf{x}\in\mathfrak{R}_{\mathbf{V}_l}^+ \\
            0 & \mbox{otherwise}
        \end{array}
        \right.
\end{equation}
where $\mathbf{V}_l$ is a set of $n$ linearly independent vectors, and $\mathbf{w}_l$ is a weight vector.
\end{theorem}

In practice, we can always assume that the input to neural networks, $\mathbf{x}$, is bounded. As a result, for computational convenience, we may normalize all inputs $\mathbf{x}$ into the unit hyper-cube, $U_n=[0,1]^n$. Obviously, this assumption can be easily incorporated into the above analysis by multiplying each $f_l(\mathbf{x})$ in eq.(\ref{pislif}) by $\prod_{r=1}^n h(x_r) h(1-x_r)$ where $h(x)$ is the Heaviside step function. Alternatively, we may simplify this term by adding $n^2$ additional hyperplanes to further split the input space to ensure all the elements of $\mathbf{x}$ do not change signs within each region $\mathfrak{R}_{\mathbf{V}_q}^+$. In this case, within each region $\mathfrak{R}_{\mathbf{V}_q}^+$, the largest absolute value among all elements of $\mathbf{x}$ is always achieved by a specific element, which is denoted as $r_q$. In other words,  the dimension $x_{r_q}$ achieves the largest absolute value inside $\mathfrak{R}_{\mathbf{V}_q}^+$.
Similarly, the normalized piece-wise linear function may be represented as a summation of some functions: $f(\mathbf{x}) =\sum_{q=1}^Q g_q(\mathbf{x})$, where each $g_q(\mathbf{x})$ $(q=1,2,\cdots,Q)$ has the following form:
\[
    g_q(\mathbf{x})=\left\{\begin{array}{ll}
        \mathbf{w}_q \cdot \mathbf{x}\ h(1-x_{r_q}) & \mathbf{x}\in\mathfrak{R}_{\mathbf{V}_q}^+ \\
        0 & \mbox{otherwise}
    \end{array}
    \right.
\]
For every $\mathbf{V}_q$,  there exists an $n\times n$ invertible matrix $\mathbf{A}_q$ to linearly transform all vectors of $\mathbf{V}_q$ into standard basis vectors $\mathbf{e}_i$ in $\mathbb{R}^n$. As a result,  
each function $g_q(\mathbf{x})$ may  be represented in terms of standard bases $\mathbf{V}_*=\{\mathbf{e}_1, \cdots, \mathbf{e}_n\}$ as follows:
\[
    g_q(\mathbf{x})=\left\{\begin{array}{ll}
        \bar{\mathbf{w}}_q \cdot \bar{\mathbf{x}}_q\ h(1-\mathbf{1} \cdot \bar{\mathbf{x}}_q) & \bar{\mathbf{x}}_q\in\mathfrak{R}^+_{\mathbf{V}_*} \\
        0 & \mbox{otherwise}
    \end{array}
    \right.
\]
where $\bar{\mathbf{x}}_q=\mathbf{x} \mathbf{A}^T_q$, and $\bar{\mathbf{w}}_q=\mathbf{w}_q\mathbf{A}^{-1}_q$.

\begin{lemma}
\label{gn}
Fourier transform of the following function:
\[ 
    s(\mathbf{x})=\left\{\begin{array}{lr}
        h(1-\mathbf{1} \cdot \mathbf{x}) & \mathbf{x}\in\mathfrak{R}^+_{\mathbf{V}_*} \\
        0 & \mbox{otherwise}
    \end{array}
    \right.
\]
may be presented as:
\begin{align}
    S(\boldsymbol{\omega})=\big(\frac{-\mathfrak{i}}{\sqrt{2\pi}}\big)^n\sum_{r=0}^n\frac{e^{-\mathfrak{i}\omega_r}}{\prod\limits_{r'\neq r}(\omega_{r'}-\omega_r)}
\end{align}
where $\omega_r$ is the $r$-th component of frequency vector $\boldsymbol{\omega}$ $(r=1,\cdots,n)$, and $\omega_0=0$.
\end{lemma}

Finally we derive the Fourier transform of fully-connected ReLU neural networks as follows. 
\begin{theorem}
\label{ft-dnn}
The Fourier transform of the output of any hidden node in a fully-connected unbiased\footnote{For mathematical convenience, we assume neural networks have no biases here. However, regular neural networks with biases may be reformulated as unbiased ones by adding another dimension of constants. Thus, the main results here are equally applicable to both cases. Note that regular neural networks with biases are used in our experiments in this paper. } ReLU neural network may be represented as 
$\sum_{q=1}^Q \mathbf{w}_q \mathbf{A}^{-1}_q  \nabla S(\boldsymbol{\omega} \mathbf{A}^{-1}_q )$,  where $\nabla$ denote the differential operator.
\end{theorem}

Obviously, neural networks are the so-called approximated bandlimited models as defined in \citep{bandlimited-19}, which have decaying high frequency components in Fourier spectrum. 
Theorem \ref{ft-dnn} further shows how the matrices $\mathbf{A}^{-1}_q$ contribute to the high frequency components when the corresponding region $\mathfrak{R}_{\mathbf{V}_q}^+$ are too small.  This is clear since the determinant of $\mathbf{A}_q$ is proportional to the volume of $\mathfrak{R}_{\mathbf{V}_q}^+$ in $\mathbb{R}^n$. As we will show later, these small regions may be explicitly exploited to generate adversarial samples for neural networks. 

\subsection{Understanding adversarial samples}
\label{SEC:Understand}

\begin{figure}[t]
	\centering
	\includegraphics[width=1\linewidth]{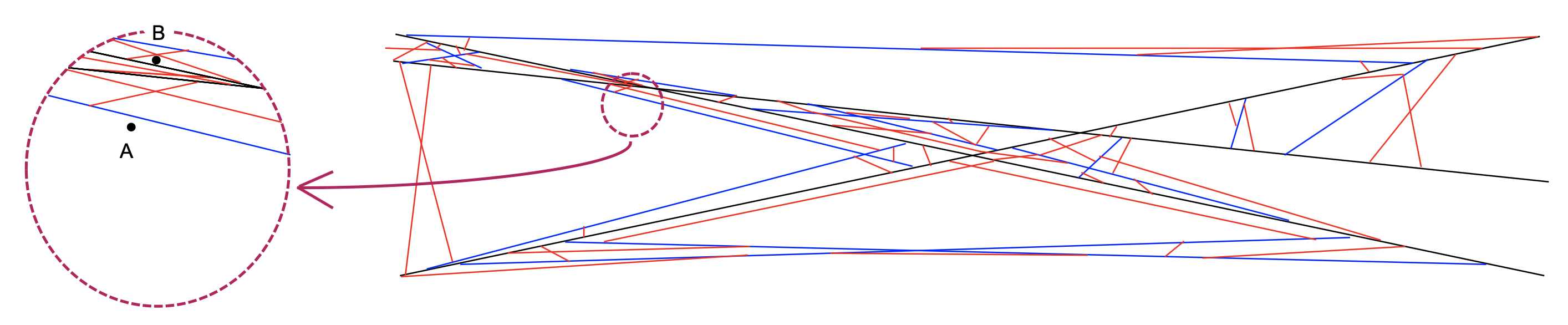}
	\caption{Illustration of input space divided into sub-regions by a biased neural network. The black lines are the hyperplanes for the first network layer, while the blue lines are for the second layer and the red lines are for the third layer. A small perturbation from point \(A\) to point \(B\) may possibly cross many hyperplanes.}
	\label{FIG:subRegions}
\end{figure}

As shown in Theorem \ref{dnn-pwlinear},  neural network may be viewed as a sequential division of the input space into many small regions, as illustrated in Figure~\ref{FIG:subRegions}. Each layer is a further division of the existing regions from the previous layers, with each region being divided differently. Hence a neural network with multiple layers would result in a tremendous amount of sub-regions in the input space. For example, when cutting an $n$-dimensional space using $N$ hyperplanes, the maximum number of regions may be computed as $ {N \choose 0} + {N \choose 1} + \cdots + {N \choose n}$. For a hidden layer of $N=1000$ nodes and input dimension is $n=200$, the maximum number of regions is roughly equal to $10^{200}$. In other words, even a middle-sized neural network can partition input space into a huge number of sub-regions, which can easily exceed the total number of atoms in the universe. When we learn a neural network, we can not expect there is at least one sample inside each region. For those regions that do not have any training sample, the resultant linear functions in them may be arbitrary since they do not contribute to the training objective function at all. Of course, most of these regions are extremely small in size. When we measure the expected loss function over the entire space, their contributions are negligible since the chance for a randomly sampled point to fall into these tiny regions is extremely small. However, adversarial attacking is imposing a new challenge since adversarial samples are not naturally sampled. Given that the total number of regions is huge, those tiny regions are almost everywhere in the input space. For any data point in the input space, we almost surely can find such a tiny region in proximity where the linear function is arbitrary. If a point inside this tiny region is selected, the output of the neural network may be unexpected. These tiny regions are the fundamental reason why neural networks are vulnerable to adversarial samples. 

In layered deep neural networks, the linear functions in all regions are not totally independent. If we use ${\bf v}^{(l)}$ to denote the weight matrix in layer $l$, the resultant linear weight $\mathbf{w}_k$ in eq.(\ref{pislif}) is actually the sum of all concatenated ${\bf v}^{(l)}$ along all active paths. When we make a small perturbation $\Delta {\bf x}$ to any input ${\bf x}$, the fluctuation in the output of any hidden node can be approximated represented as:
\begin{equation}
\Delta  f({\bf x}) \propto N \cdot \prod_{l}  {\bf E}\left[ \vert{\bf v}_{ij}^{(l)}\vert \right]
\end{equation}
where $N$ denotes the total number of hyperplanes to be crossed when moving ${\bf x}$ to ${\bf x}+\Delta {\bf x}$. In any practical neural network, we normally have at least tens of thousands of hyperplanes crossing the hypercube $U_n=[0,1]^n$. In other words, for any input ${\bf x}$ in a high-dimensional space, we can always move it to cross a large number of hyperplanes to enter a tiny region. When $N$ is fairly large, the above equation indicates that the output of a neural network can still fluctuate dramatically even after all weight vectors are regularized by $L_1$ or $L_2$ norm. 

At last, we believe the existence of many unlearned tiny regions is an intrinsic property of neural networks given its current model structure. Therefore, simple re-training strategies or small structure modifications will not be able to completely get rid of adversarial samples. In principle, neural networks must be strictly
bandlimited to filter out those decaying high frequency components in order to completely eliminate all adversarial samples.  We definitely need more research efforts to figure out how to do this effectively and efficiently for neural networks. 

\section{The proposed defence approach: post-averaging}

\subsection{Post-averaging}
\label{pave}

In this paper,  we propose a simple post-processing method to smooth out those high frequency components as much as possible, which relies on a simple idea similar to moving-average in one-dimensional sequential data. Instead of generating prediction merely from one data point, we use the averaged value within a small neighborhood around the data point, which is called {\em post-averaging} here. 
Mathematically, the post-averaging is computed as an integral over a small neighborhood centered at 
the input:
\begin{align}
    f_{C}(\mathbf{x})=\frac{1}{\mathbf{V}_{\!C}}\idotsint_{\mathbf{x}' \in C} \; f(\mathbf{x}-\mathbf{x}') \;  \text{d}\mathbf{x}'
\end{align}
where $\mathbf{x}$ is the input and $f(\mathbf{x})$ represents the output of the neural network, and 
$C$ denotes a small neighborhood centered at the origin and 
$\mathbf{V}_{\!C}$ denotes its volume. When we choose $C$ to be an $n$-sphere in $\mathbb{R}^n$ of radius $r$, we may simply derive the Fourier transform of $f_{C}(\mathbf{x})$ as follows:
\begin{equation}
    \label{ftpave}
    F_C(\boldsymbol{\omega}) =F(\boldsymbol{\omega})\frac{1}{\mathbf{V}_{\!C}}\idotsint_{\mathbf{x}'\in C} e^{-\mathfrak{i} 
    \mathbf{x}' \cdot \boldsymbol{\omega}}\ \text{d}\mathbf{x'}\\
    =F(\boldsymbol{\omega})\frac{\Gamma(\frac{n}{2}+1)}{\pi^{\frac{n}{2}}}\frac{J_{\frac{n}{2}}(r|\boldsymbol{\omega}|)}{(r|\boldsymbol{\omega}|)^{\frac{n}{2}}}
\end{equation}
where $J_{\frac{n}{2}}(\cdot)$ is the first kind Bessel function of order $n/2$. 
Since the Bessel functions, $J_{\nu}(\omega)$, decay with rate $1/\sqrt{\omega}$ as $|\omega| \to \infty$~\citep{watson1995treatise}, we have 
$F_C(\boldsymbol{\omega}) \sim  \frac{F(\boldsymbol{\omega})}{(r|\boldsymbol{\omega}|)^{\frac{n+1}{2}}}$ as $|\boldsymbol{\omega}| \to \infty$. Therefore, if $r$ is chosen properly, the post-averaging operation can significantly bandlimit neural networks by smoothing out high frequency components. Note that the similar ideas have been used in \citep{vbpc99,neighborhood03} to improve robustness in speech recognition.  

\subsection{Sampling methods}

However, it is intractable to compute the above integral for any meaningful neural network used in practical applications. In this work, we propose to use a simple numerical method to approximate it. For any input $\mathbf{x}$, we select $K$ points in the neighborhood $C$ centered at $\mathbf{x}$, i.e. $\{ \mathbf{x}_1, \mathbf{x}_2, \cdots, \mathbf{x}_K \}$ , the integral is approximately computed as 
\begin{equation}
\label{eq-numerical-integral}
f_{C}(\mathbf{x}) \approx \frac{1}{K} \sum_{k=1}^K \; f(\mathbf{x}_k).
\end{equation}

Obviously, in order to defend against adversarial samples, it is important to have samples outside the current unlearned tiny region. In the following, we use a simple sampling strategy based on directional vectors.
To generate a relatively even set of samples for eq.(\ref{eq-numerical-integral}), we first determine some directional vectors $\hat{\mathbf{v}}$, and then move the input $\mathbf{x}$ along these directions using several step sizes within the sphere of radius $r$:
\begin{equation}
\label{Eq:sampling}
    \mathbf{x}' = \mathbf{x} + \lambda \cdot \hat{\mathbf{v}}
\end{equation}
where $\lambda = [\pm \frac{r}{3}, \pm \frac{2r}{3}, \pm r]$, and  \(\hat{\mathbf{v}}\) is a  selected unit-length directional vector. For each selected direction, we generate six samples within C along both the positive and the negative directions to ensure efficiency and even sampling.
Here, we propose two different methods to sample directional vectors:

\begin{itemize}
	\item \textbf{random}: Random sampling is the simplest and most efficient method that one can come up with. We fill the directional vectors with random numbers generated from a standard normal distribution, and then normalize them to have unit length.
	
	\item \textbf{approx}: Instead of using random directions, it would be much more efficient to move out of the original region if we use the normal directions of the closest hyperplanes. In ReLU neural networks, each hidden node represents a hyperplane in the input space. For any input $\mathbf{x}$, the distance to each hyperplane may be computed as 
	$d_k^{(n)} = \frac{a_k^{(n)}}{\| \hat{\mathbf{v}}_k^{(n)} \|}$, where $a_k^{(n)}$ denotes the output of the corresponding hidden node and $\hat{\mathbf{v}}_k^{(n)} = \nabla_{\mathbf{x}} a_k^{(n)}$.
	Based on all distances computed for all hidden nodes, we can select the normal directions for the $K$ closest hyperplanes. However, computing the exact distances is computationally expensive as it requires back-propagation for all hidden nodes. In implementation, we simply estimate relative distances among all hidden units in the same layer using the weights matrix of this layer and select some closest hidden units in each layer based on the relative distances. In this way, we only need to  back-propagate for the selected units. We refer this implementation as "\textit{approx}" in the experimental results.
\end{itemize}

\section{Experiments}

In this section, we evaluate the above post-averaging method to defend against several popular adversarial attacking methods on the challenging ImageNet task. 

\subsection{Experimental setup}
\begin{itemize}
    \item \textbf{Dataset}: Since our proposed post-averaging method does not need to re-train neural networks, we do not need to use any training data in our experiments. For the evaluation purpose, we use the validation set of the ImageNet task~\citep{imageNet}. The validation set consists of 50000 images labelled into 1000 categories. Following settings in~\citep{sampleImageNet_1, sampleImageNet_2, randomLayer, defeated_4}, for computational efficiency, we randomly choose 5000 images from the ImageNet validation set and evaluate our approach on these 5000 images.
    \item \textbf{Target model}: We use a pre-trained VGG16 network~\citep{VGG} with batch normalization that is available from PyTorch. In our experiments, we directly use this pre-trained model without any modification.
    \item \textbf{Source of adversarial attacking methods}: We use Foolbox~\citep{foolbox}, an open source tool box to generate adversarial samples using different adversarial attacking methods.
    In this work, we have chosen four most popular attacking methods used in the literature: Fast Gradient Sign method (FGSM)~\citep{FGSM}, Projected Gradient Descent (PGD) method~\citep{BIM, PGD}, DeepFool (DF) attack method~\citep{DeepFool} and Carlini \& Wagner (C\&W) L2 attack method~\citep{CW}.
    \item \textbf{Threat model}: Inspired by~\citep{advSurvey, principle}, we adopt the commonly assumed threat model that: the adversarial can only attack during testing and they have complete knowledge of the target model. We also constrain the allowed perturbations by \(l_\infty\) norm \(\epsilon = \sfrac{8}{255}\).
\end{itemize}

\subsection{Evaluation criteria}
For each experiment, we define:
\begin{itemize}
    \item \textbf{Clean set}: The dataset that consists of the 5000 images randomly sampled from ImageNet.
    \item \textbf{Attacked set}: For every correctly classified image in the Clean set, if an adversarial sample is successfully generated under the attacking criteria, the original sample is replaced with the adversarial sample; if no adversarial sample is found, the original sample is kept in the dataset. Meanwhile, all the misclassified images are kept in the dataset without any change. Therefore the dataset also has 5000 images.
\end{itemize}

In our experiments, we evaluate the original model and the model defended using post-averaging on both the Clean and the Attacked sets. The performance is measured in terms of :
\begin{itemize}
    \item \textbf{Accuracy}: number of correctly classified images over the whole dataset.
    \item \textbf{Defence rate}: number of successfully defended adversarial samples over the total number of adversarial samples in the Attacked set. By "successfully defended", it refers to the case where an adversarial sample is correctly classified after the original model is defended by the post-averaging approach.
\end{itemize}

\subsection{Experimental results on top-1-miss criterion}

\begin{table}[t]
    \centering
    \caption{Performance of post-averaging against different top-1-miss attacking methods on ImageNet ($\epsilon=\sfrac{8}{255}$, $r=30$, and $K=60$).}
    \begin{tabular}[c]{p{4cm} c c c c c c}
        \hline
        & \multicolumn{2}{c}{Original Model} & \multicolumn{3}{c}{Defended by Post-Averaging} & \\ \hline
        & \multicolumn{2}{c}{Top-1 Accuracy} & \multicolumn{2}{c}{Top-1 Accuracy} & Defence & \\
        attack, defence & Clean & Attacked & Clean & Attacked & Rate & \#Adv \\ \hline
        FGSM, random & \multirow{2}{*}{0.7252} & \multirow{2}{*}{0.0224} & 0.7192 & 0.6958 & {\bf 0.9363} & \multirow{2}{*}{3514} \\
        FGSM, approx & & & 0.6786 & 0.6388 & 0.8372 & \\
        \hline
        PGD, random & \multirow{2}{*}{0.7252} & \multirow{2}{*}{0.0010} & 0.7190 & 0.7048 & {\bf 0.9508} & \multirow{2}{*}{3621}\\
        PGD, approx & & & 0.6786 & 0.6540 & 0.8630 & \\
        \hline
        DF, random & \multirow{2}{*}{0.7252} & \multirow{2}{*}{0.0120} & 0.7180 & 0.7052 & {\bf 0.9521} & \multirow{2}{*}{3571} \\
        DF, approx & &  & 0.6786 & 0.6578 & 0.8681 & \\
        \hline
        C\&W, random & \multirow{2}{*}{0.7252} & \multirow{2}{*}{0.0012} & 0.7188 & 0.7064 & {\bf 0.9533} & \multirow{2}{*}{3620} \\
        C\&W, approx & & & 0.6786 & 0.6600 & 0.8713 & \\
        \hline
    \end{tabular}
    \label{TAB:main_result}
\end{table}

 In the experiments reported in this subsection, we generated adversarial samples based on the top-1-miss criterion, which defines adversarial samples as images whose predicted classes are not the same as their true labels.

\begin{table}[t]
    \centering
    \caption{Performance of post-averaging on the sizes of used neighborhood ($\epsilon=\sfrac{8}{255}$, $K=60$).}
    \begin{tabular}[c]{p{4.2cm} c c c c c c}
        \hline
         & \multicolumn{2}{c}{Original Model} & \multicolumn{3}{c}{Defended by Post-Averaging} & \\ \hline
         & \multicolumn{2}{c}{Top-1 Accuracy} & \multicolumn{2}{c}{Top-1 Accuracy} & Defence & \\
        attack, defence & Clean & Attacked & Clean & Attacked & Rate & \#Adv \\ \hline
        FGSM, random(r=4) & \multirow{3}{*}{0.7252} & \multirow{3}{*}{0.0224} & 0.7238 & 0.4888 & 0.6614 & \multirow{3}{*}{3514} \\
        FGSM, random(r=15) & & & 0.7242 & 0.6860 & 0.9303 & \\
        FGSM, random(r=30) & & & 0.7192 & 0.6958 & 0.9363 & \\
        \hline
        FGSM, approx(r=4) & \multirow{3}{*}{0.7252} & \multirow{3}{*}{0.0224} & 0.7246 & 0.6534 & 0.8739 & \multirow{3}{*}{3514} \\
        FGSM, approx(r=15) & & & 0.7056 & 0.6534 & 0.8617 & \\
        FGSM, approx(r=30) & & & 0.6786 & 0.6388 & 0.8372 & \\
        \hline
    \end{tabular}
    \label{TAB:step_sizes}
\end{table}

Table~\ref{TAB:main_result} shows the performance of our defence approach against different attacking methods. In this table, the samples for post-averaging are selected within an $n$-sphere of radius $r=30$ as in eq.(\ref{Eq:sampling}), with $K=60$ different directions. For the \textit{approx} sampling method, we select 20 directional vectors from each of the last three fully-connected layers in the original VGG16 model, while for \textit{random} sampling we simply randomly generate 60 different directions. Both methods result in a total of \(60 \times 2 \times 3 + 1 = 361\) samples (including the input) for each input image to be used in eq.(\ref{eq-numerical-integral}). Moreover, all the adversarial samples generated are restricted to be within the perturbation range \(\epsilon=\sfrac{8}{255}\). We show the top-1 accuracy of the original model and the defended model on both the Clean and the Attacked set respectively, as well as the defence rate of the defended model. Besides, we also show the number of adversarial samples successfully generated by each attacking method in the last column.

From Table~\ref{TAB:main_result}, we can see that our proposed defence approach is universally robust to all of the attacking methods we have examined. It has achieved about 85\% defence rates in all the experiments with only a minor performance degradation in the Clean set. Especially when using the \textit{random} sampling method, our method can defend about 95\% adversarial samples while having very little performance degradation in the Clean set (less than 1\%). This is due to \textit{random} sampling can provide more evenly distributed sampling directions than \textit{approx} sampling when the neighborhood is large enough ($r=30$).

However, when the used neighborhood is small, we may anticipate that \textit{random} sampling may be more sensitive to the neighborhood size $r$. In this case, the randomly sampled directions are usually not the normal directions of the closest hyperplanes so that small radius may not be sufficient to move out the current region. To investigate this problem, we have tested both sampling methods on 3 different radii. Experimental results are shown in Table~\ref{TAB:step_sizes}. As we can see, the defence rates drop significantly from 94\% to 66\% for \textit{random} sampling when a smaller radius is used while for \textit{approx} sampling it can even get slight performance improvement when a smaller radius is used. Therefore, we recommend to use relatively larger radii for \textit{random} sampling and relatively smaller radii for \textit{approx} sampling. Moreover, we may improve and stabilize the performance by combining two methods with an ensemble model, which will be left for future investigation.

At the end, we have also investigated the effect on performance when using different numbers of sampling directions. As shown in Table~\ref{TAB:num_directions}, the defence performance doesn't vary much when much less sampling directions are used. For example, our defence approach still retains very good performance even when $K=6$ is used, in which only $6 \times 2 \times 3 + 1 = 37$ samples are evaluated for each input image. These samples can be easily packed into a mini-batch for very fast computation in GPUs. Hence when time efficiency is a concern, we can significantly reduce the number of sampling directions for faster defensive evaluation.

\begin{table}[t]
    \centering
    \caption{Performance of post-averaging on the number of sampling directions ($\epsilon=\sfrac{8}{255}$, $r=30$).}
    \begin{tabular}[c]{p{4.2cm} c c c c c c}
        \hline
         & \multicolumn{2}{c}{Original Model} & \multicolumn{3}{c}{Defended by Post-Averaging} & \\ \hline
         & \multicolumn{2}{c}{Top-1 Accuracy} & \multicolumn{2}{c}{Top-1 Accuracy} & Defence & \\
        attack, defence & Clean & Attacked & Clean & Attacked & Rate & \#Adv \\ \hline
        FGSM, random(K=6) & \multirow{3}{*}{0.7252} & \multirow{3}{*}{0.0224} & 0.7180 & 0.6944 & 0.9351 & \multirow{3}{*}{3514} \\
        FGSM, random(K=15) & & & 0.7194 & 0.6948 & 0.9351 & \\
        FGSM, random(K=60) & & & 0.7192 & 0.6958 & 0.9363 & \\
        \hline
        FGSM, approx(K=6) & \multirow{3}{*}{0.7252} & \multirow{3}{*}{0.0224} & 0.6754 & 0.6128 & 0.7999 & \multirow{3}{*}{3514} \\
        FGSM, approx(K=15) & & & 0.6802 & 0.6280 & 0.8190 & \\
        FGSM, approx(K=60) & & & 0.6786 & 0.6388 & 0.8372 & \\
        \hline
    \end{tabular}
    \label{TAB:num_directions}
\end{table}

\subsection{Experimental results on top-5-miss criterion}

For image classification on the ImageNet task, it is usually more reasonable to use top-5 accuracy due to large number of confusing classes and multi-label cases, in this subsection, we have also evaluated our defence approach against adversarial samples that are generated based on the top-5-miss criterion. Under the top-5-miss criterion, adversarial samples are defined as images whose true labels are not among their top 5 predictions. Note that although the adversarial samples are easier to defend under the top-5-miss criterion, the adversarial samples generated are actually much stronger since the true labels are pushed out of the top-5 predictions. Experimental results are shown in Table~\ref{TAB:top5_results}. As shown in the table, our defence approach using \textit{random} sampling can achieve over 97\% defence rates against all four attacking methods.  Meanwhile, when measured by top-5 accuracy, we can see that the defended models using \textit{random} sampling yield almost no performance degradation in the Clean set and achieves only a small performance degradation (about 1-3\%) in the Attacked set. 

\begin{table}[t]
    \centering
    \caption{Performance of post-averaging against different top-5-miss attacking methods on ImageNet ($\epsilon=\sfrac{8}{255}$, $r=30$, and $K=60$).}
    \resizebox{1\textwidth}{!}{
    \begin{tabular}[c]{p{3.2cm} c c c c c c c c c c}
        \hline
        & \multicolumn{4}{c}{Original Model} & \multicolumn{5}{c}{Defended by Post-Averaging} & \\ \hline
         & \multicolumn{2}{c}{Top-1 Accuracy} & \multicolumn{2}{c}{Top-5 Accuracy} & \multicolumn{2}{c}{Top-1 Accuracy} &  \multicolumn{2}{c}{Top-5 Accuracy} & Defence & \\
        attack,defence & Clean & Attacked & Clean & Attacked & Clean & Attacked & Clean & Attacked & Rate & \#Adv \\ \hline
        FGSM, random & \multirow{2}{*}{0.7252} & \multirow{2}{*}{0.1306} & \multirow{2}{*}{0.9136} & \multirow{2}{*}{0.1544} & 0.7184 & 0.3436 & 0.9106 & 0.8892 & {\bf 0.9565} & \multirow{2}{*}{3796} \\
		FGSM, approx & & & & & 0.6786 & 0.3874 & 0.8856 & 0.7914 & 0.8177 & \\
		\hline
		PGD, random & \multirow{2}{*}{0.7252} & \multirow{2}{*}{0.0074} & \multirow{2}{*}{0.9136} & \multirow{2}{*}{0.0096} & 0.7190 & 0.3172 & 0.9112 & 0.9014 & {\bf 0.9768} & \multirow{2}{*}{4520}\\
		PGD, approx & & & & & 0.6786 & 0.3902 & 0.8856 & 0.8290 & 0.8883 & \\
		\hline
		DF, random & \multirow{2}{*}{0.7252} & \multirow{2}{*}{0.0454} & \multirow{2}{*}{0.9136} & \multirow{2}{*}{0.0534} & 0.7180 & 0.5006 & 0.9110 & 0.9034 & {\bf 0.9788} & \multirow{2}{*}{4301} \\
		DF, approx & & & & & 0.6786 & 0.5528 & 0.8856 & 0.8642 & 0.9247 & \\
		\hline
		C\&W, random & \multirow{2}{*}{0.7252} & \multirow{2}{*}{0.1116} & \multirow{2}{*}{0.9136} & \multirow{2}{*}{0.3000} & 0.7196 & 0.3214 & 0.9116 & 0.9054 & {\bf 0.9889} & \multirow{2}{*}{3068} \\
		C\&W, approx & & & & & 0.6786 & 0.4294 & 0.8856 & 0.8458 & 0.9263 & \\
        \hline
    \end{tabular}}
    \label{TAB:top5_results}
\end{table}

\section{Final remarks}

In this paper, we have presented some theoretical results on Fourier analysis of ReLU neural networks. These results are useful for us to understand why neural networks are vulnerable to adversarial samples. As a possible defence strategy, we have proposed a simple post-averaging method. Experimental results on ImageNet have demonstrated that our simple defence technique turns to be very effective against many popular attack methods in the literature. Finally, it will be interesting to see whether our post-averaging method will be still robust against any new attack methods in the future. 

\section*{Acknowledgments}
%
%
This work is supported partially by a research donation from iFLYTEK Co., Ltd., Hefei, China, and a discovery grant from Natural Sciences and Engineering Research Council (NSERC) of Canada.

\bibliography{PostStageDefense}
\bibliographystyle{plainnat}

\clearpage
\newpage

\section*{Appendix: Mathematical proofs}

\setcounter{section}{2}

\begin{definition}
    A piece-wise linear function is a continuous function $f:\mathbb{R}^n \xrightarrow{} \mathbb{R}$ such that there are some hyperplanes passing through origin and dividing $\mathbb{R}^n$ into $M$ pairwise disjoint regions $\mathfrak{R}_m$, $(m=1,2,...,M)$, on each of which $f$ is linear:
$$  
        f(\bold{x})=\left\{\begin{array}{ll}
            \bold{w}_1 \cdot \bold{x} & \bold{x}\in\mathfrak{R}_1 \\
            \bold{w}_2 \cdot \bold{x} & \bold{x}\in\mathfrak{R}_2 \\
            \vdots\\
            \bold{w}_M \cdot \bold{x} & \bold{x}\in\mathfrak{R}_M
            \end{array} \right.
$$  
\end{definition}
\begin{lemma}
\label{lemre-app}
Composition of a piece-wise linear function with a ReLU activation function is also a piece-wise linear function.
\end{lemma}
\begin{proof}
Let r(.) denote the ReLU activation function. If $f(\bold{x})$ on region $\mathfrak{R}_m$ takes both positive and negative values, $r\big(f(\bold{x})\big)$ will break it into two regions $\mathfrak{R}_p^+$ and $\mathfrak{R}_p^0$. On the former $r\big(f(\bold{x})\big)=f(\bold{x})$ and on the latter $r\big(f(\bold{x})\big)=0$, which both are linear functions. As $f(\bold{x})$ on $\mathfrak{R}_p$ is linear, common boundary of $\mathfrak{R}_p^+$ and $\mathfrak{R}_p^0$ lies inside a hyperplane passing through origin -- which is the kernel of the linear function. Therefore, if $f(\bold{x})$ is a piece-wise linear function defined by $k$ hyperplanes resulting in $M$ regions, $r\big(f(\bold{x})\big)$ will be a piece-wise linear function defined by at most $k+m$ hyperplanes.
\end{proof}

\begin{theorem}
\label{dnn-pwlinear-app}
    The output of any hidden unit in an unbiased fully-connected ReLU neural network is a piece-wise linear function.
\end{theorem}
\begin{proof}
This proposition immediately follows lemma \ref{lemre-app}.
\end{proof}
\begin{definition}
    Let $\bold{V}=\{{\bf v}_1, {\bf v}_2,..., {\bf v}_n\}$ be a set of $n$ independent vectors in $\mathbb{R}^n$. An infinite simplex, $\mathfrak{R}_\bold{V}^+$, is defined as the region linearly spanned by $\bold{V}$ using only positive weights:
    \begin{align}
        \mathfrak{R}_\bold{V}^+=\{\sum_{k=1}^n\ \alpha_k {\bf v}_k\ |\ \forall k\ \alpha_k>0\}
    \end{align}
\end{definition}
\begin{theorem}
    \label{suire-app}
    Each piece-wise linear function $f(\bold{x})$ can be formulated as a summation of some functions: $ f(\bold{x}) = \sum_{k=1}^K f_k(\bold{x})$, each of which is linear and non-zero only in an infinite simplex as follows:
$$ 
        \label{pislif-app}
        f_k(\bold{x})=\left\{\begin{array}{lr}
            \bold{w}_k \cdot \bold{x} & \bold{x}\in\mathfrak{R}_{\bold{V}_k}^+ \\
            0 & \mbox{otherwise}
        \end{array}
        \right.
$$  
    where $\bold{V}_k$ is a set of $n$ independent vectors, and $\bold{w}_k$ is a weight vector.
\end{theorem}
\begin{proof}
Each region $\mathfrak{R}_p$ of a piece-wise linear function, $f(\bold{x})$, which describes the behavior of a ReLU node if intersects with an affine hyper-plane results in a convex polytope. This convex polytope can be triangulated into some simplices. Define $\bold{V}_k$, $(k=1,2,...,K)$, sets of vertexes of these simplices. The infinite simplexes created by these vector sets will have the desired property and $f(\bold{x})$ can be written as:
$
    \label{ftofij-app}
    f(\bold{x}) = \sum_{k=1}^K f_k(\bold{x})
$.
\end{proof}

As explained earlier in the original article by adding $n^2$ hyper-planes to those defining the piece-wise linear function, the output of a ReLU node may be represented as $f(\mathbf{x}) =\sum_{q=1}^Q g_q(\mathbf{x})$. These hyper-planes are those perpendicular to standard basis vectors and subtraction of one of these vectors from another one. That is, $\mathbf{e}_i\ (i=1,\dots, n)$ and $\mathbf{e}_i-\mathbf{e}_j\ (1\leq i<j\leq n)$. Given this representation, the final step to achieve the Fourier transform is the following lemma:
\begin{lemma}
\label{gn-app}
Fourier transform of the following function:
$$
    s(\bold{x})=\left\{\begin{array}{lr}
        h(1-\bold{1} \cdot \bold{x}) & \bold{x}\in\mathfrak{R}^+_{\bold{V}_*} \\
        0 & \mbox{otherwise} 
    \end{array}
    \right.
$$

may be presented as:
\begin{align}
    S(\boldsymbol{\omega})=\big(\frac{-\mathfrak{i}}{\sqrt{2\pi}}\big)^n\sum_{r=0}^n\frac{e^{-\mathfrak{i}\omega_r}}{\prod\limits_{r'\neq r}(\omega_{r'}-\omega_r)}
\end{align}
where $\omega_r$ is the $r$th component of frequency vector $\boldsymbol{\omega}$ $(r=1,\cdots,n)$, and $\omega_0=0$.
\end{lemma}
\begin{proof}
Alternatively, $s(\bold{x})$ may be represented as:
\begin{align}
    s(\bold{x})=h(\bold{1} \cdot \bold{x})h(1-\bold{1} \cdot \bold{x})\prod_{j=1}^n h(x_j)h(1-x_j)
\end{align}
Therefore, we need to compute Fourier transform of $h(x)h(1-x)$:
\begin{align}
    \frac{1}{\sqrt{2\pi}}\int_{-\infty}^\infty e^{-\mathfrak{i}x\omega} h(x)(1-x) \text{d}x&=\frac{1}{\sqrt{2\pi}}\int_0^1 e^{-\mathfrak{i}x\omega}\mathrm{d}x\\
    &=\frac{-\mathfrak{i}}{\sqrt{2\pi}}\frac{1-e^{-\mathfrak{i}\omega}}{\omega}
\end{align}
By taking the inverse Fourier transform of the function:
\begin{align}
    (\sqrt{2\pi})^{n-1}\int_{-\infty}^\infty\frac{-\mathfrak{i}}{\sqrt{2\pi}}\frac{1-e^{-\mathfrak{i}\zeta}}{\zeta}\ \boldsymbol{\delta}_n(\boldsymbol{\omega}-\zeta\bold{1})\ \mathrm{d}\zeta
\end{align}
where $\boldsymbol{\delta}_n$ is $n$-dimensional Dirac Delta function, it can be shown that it is the Fourier transform of $h(\bold{1} \cdot \bold{x})h(1-\bold{1} \cdot \bold{x})$:
\begin{align}
    (\frac{1}{\sqrt{2\pi}})^n\int\dots\int_{\mathbb{R}^n}e^{i\boldsymbol{\omega} .\bold{x}}&(\sqrt{2\pi})^{n-1}\int_{-\infty}^\infty\frac{-\mathfrak{i}}{\sqrt{2\pi}}\frac{1-e^{-\mathfrak{i}\zeta}}{\zeta}\ \boldsymbol{\delta}_n(\boldsymbol{\omega}-\zeta\bold{1})\ \mathrm{d}\zeta\ \mathrm{d}\boldsymbol{\omega}\\
    &=\frac{1}{\sqrt{2\pi}}\idotsint_{\mathbb{R}^n} e^{i\boldsymbol{\omega} .\bold{x}} \int_{-\infty}^\infty\frac{-\mathfrak{i}}{\sqrt{2\pi}}\frac{1-e^{-\mathfrak{i}\zeta}}{\zeta}\ \boldsymbol{\delta}_n(\boldsymbol{\omega}-\zeta\bold{1})\ \mathrm{d}\zeta\ \mathrm{d}\boldsymbol{\omega}\\
    &=\frac{1}{\sqrt{2\pi}}\int_{-\infty}^\infty\frac{-\mathfrak{i}}{\sqrt{2\pi}}\frac{1-e^{-\mathfrak{i}\zeta}}{\zeta}\idotsint_{\mathbb{R}^n} e^{i\boldsymbol{\omega} .\bold{x}}\boldsymbol{\delta}_n(\boldsymbol{\omega}-\zeta\bold{1})\ \mathrm{d}\boldsymbol{\omega}\ \mathrm{d}\zeta\\
    &=\frac{1}{\sqrt{2\pi}}\int_{-\infty}^\infty\frac{-\mathfrak{i}}{\sqrt{2\pi}}\frac{1-e^{-\mathfrak{i}\zeta}}{\zeta}\ e^{i\zeta\bold{1} .\bold{x}}\ \mathrm{d}\zeta\\
    &=h(\bold{1} \cdot \bold{x})h(1-\bold{1} \cdot \bold{x})
\end{align}
Now we can find the Fourier transform of $s(\bold{x})$
\begin{align}
     S(\boldsymbol{\omega})&=\big(\prod_{r=1}^n\frac{-\mathfrak{i}}{\sqrt{2\pi}}\frac{1-e^{-\mathfrak{i}\omega_r}}{\omega_r}\big)*(\sqrt{2\pi})^{n-1}\int_{-\infty}^\infty\frac{-\mathfrak{i}}{\sqrt{2\pi}}\frac{1-e^{-\mathfrak{i}\zeta}}{\zeta}\ \boldsymbol{\delta}_n(\boldsymbol{\omega}-\zeta\bold{1})\ \mathrm{d}\zeta\\
     \label{ftnorli-app}
     &=\mathfrak{i}(\frac{-\mathfrak{i}}{\sqrt{2\pi}})^{n+2}\int_{-\infty}^\infty e^{-\mathfrak{i}\zeta}\prod_{r=0}^n\frac{1-e^{-\mathfrak{i}(\omega_r-\zeta)}}{\omega_r-\zeta}\ \mathrm{d}\zeta
\end{align}
where $*$ is convolution operator. The final integrand may be represented as:
\begin{align}
    e^{-\mathfrak{i}\zeta}\prod_{r=0}^n\frac{1-e^{-\mathfrak{i}(\omega_r-\zeta)}}{\omega_r-\zeta}&=e^{-\mathfrak{i}\zeta}\prod_{r=0}^n\frac{1}{\omega_r-\zeta}\prod_{r=0}^n (1-e^{-\mathfrak{i}(\omega_r-\zeta)})\\
    &=e^{-\mathfrak{i}\zeta}\sum_{r=0}^n\frac{A_r}{\omega_r-\zeta}\prod_{r=0}^n (1-e^{-\mathfrak{i}(\omega_r-\zeta)})\\
     &=e^{-\mathfrak{i}\zeta}\sum_{r=0}^n\frac{A_r}{\omega_r-\zeta}\sum_{B\subseteq \Omega}(-1)^{|B|}e^{-\mathfrak{i}(\sigma_B-|B|\zeta)}\\
     &=\sum_{r=0}^n\frac{A_r}{\omega_r-\zeta}\sum_{B\subseteq \Omega}(-1)^{|B|}e^{-\mathfrak{i}(\sigma_B-(|B|-1)\zeta)}
\end{align}
where $\Omega=\{\omega_0,...,\omega_n\}$, $\sigma_B$ is the summation over elements of $B$ and $A_r=\prod\limits_{r'\neq r}\frac{1}{\omega_{r'}-\omega_r}$. Therefore:
\begin{align}
    \int_{-\infty}^\infty e^{-\mathfrak{i}\zeta}\prod_{r=0}^n&\frac{1-e^{-\mathfrak{i}(\omega_r-\zeta)}}{\omega_r-\zeta}\ \mathrm{d}\zeta\\
    &=\int_{-\infty}^\infty\sum_{r=0}^n\frac{A_r}{\omega_r-\zeta}\sum_{B\subseteq S}(-1)^{|B|}e^{-\mathfrak{i}(\sigma_B-(|B|-1)\zeta)}\ \mathrm{d}\zeta\\
    &=\sum_{r=0}^n A_r \int_{-\infty}^\infty\frac{1}{\omega_r-\zeta}\sum_{B\subseteq S}(-1)^{|B|}e^{-\mathfrak{i}(\sigma_B-(|B|-1)\zeta)}\ \mathrm{d}\zeta\\
    &=\sum_{r=0}^n A_r \int_{-\infty}^\infty\frac{1}{\zeta}\sum_{B\subseteq S}(-1)^{|B|+1}e^{-\mathfrak{i}(\sigma_B-(|B|-1)\omega_r+(|B|-1)\zeta)}\ \mathrm{d}\zeta\\
    &=\sum_{r=0}^n A_r \sum_{B\subseteq S}(-1)^{|B|}\mathfrak{i}\pi\  \mbox{sign}(|B|-1)e^{-\mathfrak{i}(\sigma_B-(|B|-1)\omega_r)}
\end{align}
If $B$ does not contain $\omega_r$ and have at least 2 elements then the terms for $B$ and $B\cup\{\omega_r\}$ will cancel each other out. Also, $\mbox{sign}(|B|-1)$ will vanish if $B$ has only one element. Therefore, there only remains empty set and sets with two elements one of them being $\omega_r$. Given the fact that $\sum A_r=0$, the result of the integral will be:
\begin{align}
    \int_{-\infty}^\infty e^{-\mathfrak{i}\zeta}\prod_{r=0}^n\frac{1-e^{-\mathfrak{i}(\omega_r-\zeta)}}{\omega_r-\zeta}\ \mathrm{d}\zeta&=\mathfrak{i}\pi\sum_{r=0}^nA_r(-e^{-\mathfrak{i}\omega_r}+\sum_{r'\neq r}e^{-\mathfrak{i}\omega_{r'}})\\
    &=-2\mathfrak{i}\pi\sum_{r=0}^n A_r e^{-\mathfrak{i}\omega_r}
    \label{intre-app}
\end{align}
Finally, substituting \ref{intre-app} into \ref{ftnorli-app} yields to the desired result.
\end{proof}

\begin{theorem}
\label{ft-dnn-app}
The Fourier transform of the output of any hidden node in a fully-connected ReLU neural network may be represented as $\sum_{q=1}^Q \bold{w}_q\bold{A}^{-1}_q  \nabla S(\boldsymbol{\omega}\bold{A}^{-1}_q)$,  where $\nabla$ denote the differential operator.
\end{theorem}
\begin{proof}
As discussed in the original paper, $f(\mathbf{x}) =\sum_{q=1}^Q g_q(\mathbf{x})$ where:
\begin{align}
    g_q(\mathbf{x})=\left\{\begin{array}{ll}
        \bar{\mathbf{w}}_q \cdot \bar{\mathbf{x}}_q\ h(1-\mathbf{1} \cdot \bar{\mathbf{x}}_q) & \bar{\mathbf{x}}_q\in\mathfrak{R}^+_{\mathbf{V}_*} \\
        0 & \mbox{otherwise}
    \end{array}
    \right.
\end{align}
or equivalently:
\begin{align}
    g_q(\bold{x})=\bar{\mathbf{w}}_q \cdot \bar{\mathbf{x}}_q s(\bar{\bold{x}}_q)
\end{align}
Therefore:
\begin{align}
    F(\boldsymbol{\omega})&=\sum_{q=1}^Q G_q(\boldsymbol{\omega})\\
    &=\sum_{q=1}^Q \bar{\mathbf{w}}_q.\nabla S(\bar{\boldsymbol{\omega}}_q)
\end{align}
where $\bar{\boldsymbol{\omega}}_q=\boldsymbol{\omega}\bold{A}^{-1}_q$.
\end{proof}

\subsection*{Derivation of eq.(\ref{ftpave})}

As for the Fourier transform computed in section \ref{pave}, it should be mentioned that the integral in equation \ref{ftpave} is the Fourier transform of:
\begin{align}
    h_r(\mathbf{x})=h(r-|\mathbf{x}|)
\end{align}
which can be derived utilizing the property of the Fourier transforms for radially symmetric functions~\citep{stein1971introduction}:
\begin{align}
    H_r(\boldsymbol{\omega})&=|\boldsymbol{\omega}|^{-\frac{n-2}{2}}\int_0^\infty J_\frac{n-2}{2}(|\boldsymbol{\omega}|\rho)\rho^\frac{n-2}{2}h(r-\rho)\rho\ \text{d}\rho\\
    &=|\boldsymbol{\omega}|^{-\frac{n-2}{2}}\int_0^r J_\frac{n-2}{2}(|\boldsymbol{\omega}|\rho)\rho^\frac{n}{2}\ \text{d}\rho\\
    &=(\frac{r}{|\boldsymbol{\omega}|})^\frac{n}{2} J_\frac{n}{2}(r|\boldsymbol{\omega}|)
\end{align}
Given this transform:
\begin{align}
    F_C(\boldsymbol{\omega}) &=F(\boldsymbol{\omega})\frac{1}{\mathbf{V}_{\!C}}\idotsint_{\mathbf{x}'\in C} e^{-\mathfrak{i} 
    \mathbf{x}' \cdot \boldsymbol{\omega}}\ \text{d}\mathbf{x'}\\
    &=F(\boldsymbol{\omega})\frac{\Gamma(\frac{n}{2}+1)}{\pi^{\frac{n}{2}}}\frac{J_{\frac{n}{2}}(r|\boldsymbol{\omega}|)}{(r|\boldsymbol{\omega}|)^{\frac{n}{2}}}
\end{align}

\end{document}